\documentclass[conference]{IEEEtran}
\IEEEoverridecommandlockouts
\usepackage{cite}
\usepackage{amsmath,amssymb,amsfonts}
\usepackage{textcomp}
\usepackage{algorithmic}
\usepackage{graphicx}
\usepackage{textcomp}
\usepackage{xcolor}
\usepackage{graphicx} 
\usepackage{amsmath}
\usepackage{amsfonts}
\usepackage{amsthm}
\usepackage{amssymb}
\usepackage{tikz}
\usepackage{amsmath,amsthm,amssymb}
\usepackage{comment} 
\usepackage[shortlabels]{enumitem}

\newcommand{\relu}{h} 
\newcommand{\Var}{\mathop{\mathrm{Var}}}
\newcommand{\Cov}{\mathop{\mathrm{Cov}}}
\newcommand{\Cor}{\mathop{\mathrm{Corr}}}

\newcommand{\Normal}{\mathcal{N}}
\newcommand{\coloneq}{\mathrel{{:}{=}}}
\newcommand{\mat}[1]{\mathbf{#1}}
\renewcommand{\vec}[1]{\mathbf{#1}}
\newcommand{\indexset}[2][0]{\{#1..#2\}}
\newcommand{\Lin}{\mathrm{L}}
\newcommand{\Exp}{\mathop{\mathbb{E}}}
\newcommand{\indicator}[1]{\mathbf{1}\{#1\}}
\newcommand{\BigO}{\mathop{{O}}}
\newcommand{\littleo}{\mathop{{o}}}

\newtheorem{theorem}{Theorem}[section] 
\newtheorem{lemma}[theorem]{Lemma}     
\newtheorem{corollary}[theorem]{Corollary}
\newtheorem{proposition}[theorem]{Proposition}
\newtheorem{definition}[theorem]{Definition}

\theoremstyle{remark}

\IEEEpubid{\makebox[\columnwidth]{979-8-3315-5937-3/25/\$31.00~\copyright~2025 IEEE \hfill} \hspace{\columnsep}\makebox[\columnwidth]{}}

\begin{document}

\title{Complexity of One-Dimensional ReLU DNNs
\thanks{Research was sponsored by the Department of the Air Force Artificial
Intelligence Accelerator and was accomplished under Cooperative Agreement
Number FA8750-19-2-1000. The views and conclusions contained in this
document are those of the authors and should not be interpreted as representing
the official policies, either expressed or implied, of the Department of the
Air Force or the U.S. Government. The U.S. Government is authorized to
reproduce and distribute reprints for Government purposes notwithstanding
any copyright notation herein.}
}

\author{\IEEEauthorblockN{Jonathan Kogan, Hayden Jananthan, Jeremy Kepner}
\IEEEauthorblockA{\textit{Massachusetts Institute of Technology} \\
}}

\maketitle

\begin{abstract}
    We study the expressivity of one–dimensional (1D) ReLU deep neural networks through the lens of their linear regions. For randomly initialized, fully connected 1D ReLU networks (He scaling with nonzero bias) in the infinite–width limit, we prove that the expected number of linear regions grows as $\sum_{i = 1}^L n_i + \littleo\left(\sum_{i = 1}^L{n_i}\right) + 1$, where $n_\ell$ denotes the number of neurons in the $\ell$-th hidden layer. We also propose a function-adaptive notion of sparsity that compares the expected regions used by the network to the minimal number needed to approximate a target within a fixed tolerance.
\end{abstract}


\section{Introduction}

Deep Neural Networks (DNNs) with Rectified Linear Unit (ReLU) activation functions are piecewise-linear functions whose expressive power can be studied via the number of linear regions that they create \cite{montufar2014number, serra2018bounding, raghu2017expressive}. Universal Approximation Theorems guarantee that sufficiently deep and wide ReLU networks can approximate any continuous functions on a compact domain \cite{cybenko1989approximation, leshno1993multilayer, yarotsky2017error}. However, achieving such approximations for complicated functions typically demands substantial computational resources. 

The Lottery Ticket Hypothesis states that we can often remove many connections while maintaining similar performance, motivating the study of sparse DNNs \cite{frankle2018lottery}. Among these sparse DNNs, RadiX-Nets represent a class of `de novo' sparse networks that are architectures with low connectivity at initialization that perform empirically well \cite{kepner2019radix, kwak2023testing}.

Prior work has shown the expected number of linear regions in a ReLU network scales with the number of neurons \cite{hanin2019complexity} and several bounds on the maximum and expected number of linear regions have been developed \cite{wang2022estimation, hanin2019deep,  montufar2014number, hanin2019complexity}. However, no closed-form expression exists for the expected number of linear regions. 

Existing notions of sparsity are typically defined in terms of parameter counts, pruning ratios, or architectural aspects and thus can be easily manipulated when approximating different functions. Thus, it is often not clear how pruning and sparsification techniques are actually influenced by the complexity of the function.

The aim of this paper is to address these gaps by studying the expected number of linear regions of ReLU DNNs with one-dimensional (1D) input and output. We also formulate a definition of sparsity based on linear regions which provides a measure of the effectiveness and expressiveness of a DNN.

\S\ref{sec:background} covers mathematical preliminaries and notation, including the probabilistic setup. \S\ref{sec:sparsity} proposes a new definition of sparsity which is independent of DNN architecture. \S\ref{sec:linearregions} analyzes linear regions in the 1D setting, with the proof of the main result covered in \S\ref{sec:proof}.

\section{Background}\label{sec:background}

We first cover some necessary notation. `$\coloneq$' denotes assignment or definition of a variable or notation while `$=$ denotes logical equality. $\mathbb{N}$ denotes the set of natural numbers $0, 1, 2, \ldots$ and, for $n, m \in \mathbb{N}$, $\indexset[n]{m} \coloneq \{ \ell \in \mathbb{N} \mid n \leq \ell \leq m \}$. $\mathbb{R}^{n \times m}$ denotes the set of all $n \times m$ matrices over the set of real numbers, $\mathbb{R}$. The rectified linear unit is denoted by $\relu(\cdot) \coloneq \max(\cdot, 0)$. Given a random variable $X \colon \Omega \to E$, $\Pr(X \in S)$ denotes the probability that the value of $X$ lies in $S \subseteq E$, $\Exp(X)$ denotes the expected value of $X$, and $\Var(X)$ denotes the variance of $X$. If $Y \colon \Omega \to E$ is another random variable then $\Cov(X, Y)$ denotes the covariance of $X$ and $Y$ while $\Cor(X, Y)$ denotes the correlation between $X$ and $Y$. We write $X \sim \Normal(\mu, \sigma^2)$ to denote that $X$ is a Gaussian random variable with mean $\mu$ and standard deviation $\sigma$. $\indicator{x \in S}$ denotes the indicator function for the set $S$. $\BigO$ and $\littleo$ denote Big-O and little-o notation, respectively.

Next we formalize the neural network notions that will be used throughout the paper. 

\begin{definition} \label{def:dnnarchitecture}
    A \emph{DNN topology} is a tuple of layer-widths $(n_0, n_1, \dots, n_L, n_{L+1}) \in \mathbb{N}^{L+2}$ for some $L \in \mathbb{N}$.
\end{definition}

\begin{definition} \label{def:dnn}
    Consider a DNN topology $(n_0, n_1, \dots, n_L, n_{L+1}) \in \mathbb{N}^{L+2}$. A \emph{ReLU DNN} (or simply \emph{DNN}) is a function $\Phi \colon \mathbb{R}^{n_0} \to \mathbb{R}^{n_{L+1}}$ specified by a selection of real weight‐matrices $\mat{W}_{\ell} \in \mathbb{R}^{n_{\ell} \times n_{\ell-1}}$ and bias‐vectors $\vec{b}_{\ell} \in \mathbb{R}^{n_{\ell}}$ for $\ell \in \indexset[1]{L+1}$. It is defined recursively defining
    \begin{equation*}
        \vec{x}^{(0)} \coloneq \vec{x}, \quad
        \vec{x}^{(\ell)} \coloneq \relu(\mat{W}_{\ell} \vec{x}^{(\ell-1)} + \vec{b}_{\ell}) \quad  (\ell \in \indexset[1]{L}),
    \end{equation*}
    so that $\Phi(x) \coloneq \mat{W}_{L+1} \vec{x}^{(L)} + \vec{b}_{L+1}$.
\end{definition}

We work in the infinite-width case $\min\{n_1, n_2,\dots,n_L\} \to \infty$. In this setting, a randomly initialized DNN converges to a Gaussian process (GP) \cite{lee2017deep}, so for a randomly initialized DNN $\Phi$ and fixed input $x$, the output $\Phi(x)$ and $p(x)$ is a Gaussian random variable, where $p$ is any pre-activation beyond the first layer. The random initialization process used in this paper given below is motivated by the He initialization \cite{he2015delving}. 

\begin{definition} \label{def:dnninitialization}
    A \emph{randomly initialized} DNN has probabilistically independent weights $W^{(\ell)}_{j i} \sim \Normal\left(0, \frac{2}{n_{\ell-1}}\right)$ and biases $b^{(\ell)}_j \sim \Normal\left(0,\sigma_b^2\right)$ for some fixed $\sigma_b \neq 0$.

\end{definition}

From this point forward, DNN topologies are assumed to be of the form $(1, n_1, \dots, n_L, 1)$. $T$ denotes such an arbitrary DNN topology while $\Phi$ denotes a randomly initialized ReLU DNN with DNN topology $T$. $W^{(\ell)}_{ ij}$ denotes the weight of the connection between the $i$-th neuron in layer $\ell - 1$ and the $j$-th neuron in layer $\ell$, while $b^{(\ell)}_j$ denotes the bias of the $j$-th neuron in layer $\ell$. $\sigma_b$ denotes the fixed standard deviation of the biases.

\section{Sparsity}\label{sec:sparsity}

Sparsity plays a central role in both network pruning and the reduction of overfitting. At the same time, excessively sparse DNNs risk losing sufficient expressive power. Standard notions of sparsity depend on fixed architectural hyperparameters, so can be artificially manipulated when approximating different functions. For example, stretching or reparameterizing a network may preserve functional expressivity while producing arbitrarily different sparsity values. To address this ambiguity, we introduce a new notion of sparsity grounded in the functional complexity of the network’s approximation measured by linear regions.  

Let $f \colon K \to \mathbb{R}$ be a continuous function on a compact domain $K  \subseteq \mathbb{R}$. Fix a base tolerance $\varepsilon_0 > 0$ and a slack factor $\alpha \ge 1$ so that we allow approximation error up to $\alpha \varepsilon_0$. We first define the minimal complexity of a function.

\begin{definition}[minimal linear complexity] \label{def:linearcomplexity}
    For any tolerance $\varepsilon_0 > 0$, let $\Lin_{\min}(f, \varepsilon_0)$ be the smallest number of linear regions of any continuous piecewise-linear function $g \colon K \to\mathbb{R}$ satisfying $\sup_{x \in K}{|g(x) - f(x)|} \le \varepsilon_0$.
\end{definition}

Let $\Lin(\Phi) \coloneq \#\{\text{linear regions of $\Phi|_K$}\}$ and define the region inefficiency relative to the base tolerance $\varepsilon_0$ as
\begin{equation*}
    \eta_{\mathrm{region}}(\Phi; f, \varepsilon_0)
    \coloneq \Exp[\Lin(\Phi)]/\Lin_{\min}(f, \varepsilon_0).
\end{equation*}

\begin{definition}[$(f, \varepsilon_0, \alpha, c)$-region-adaptive sparsity]
    Fix $f$, base tolerance $\varepsilon_0 > 0$, slack $\alpha \ge 0$, and inefficiency bound $c \ge 0$. A ReLU network $\Phi$ is said to be \emph{$(f, \varepsilon_0, \alpha, c)$-region-adaptively sparse} if
    \begin{enumerate}[(i)]
      \item (approximating) $\displaystyle \sup_{x \in K}{|\Phi(x) - f(x)|} \le \alpha \varepsilon_0$ and
      \item (region efficient) $\eta_{\mathrm{region}}(\Phi; f, \varepsilon_0) \le c$.
    \end{enumerate}
\end{definition}

This definition makes sparsity \emph{function-adaptive}, since it depends only on the approximation complexity of $f$, as measured by $\Lin_{\min}(f, \varepsilon_0)$, rather than on parameter counts. In particular, if two networks $\Phi$ and $\Psi$ satisfy $\Phi(x) = \Psi(x)$ for all $x \in K$, then $\eta_{\mathrm{region}}(\Phi; f, \varepsilon_0) =\eta_{\mathrm{region}}(\Psi; f, \varepsilon_0)$, so region-adaptive sparsity is invariant under architectural reparametrizations. Thus, sparsity is tied directly to expressivity via region efficiency: higher inefficiency indicates that the network is overparametrized relative to $f$, while lower inefficiency indicates that the network is using its regions effectively. The slack factor $\alpha$ further quantifies the desired accuracy of approximation.

\section{Linear Regions}\label{sec:linearregions}

In the one-dimensional setting, the number of linear regions is exactly the number of breakpoints plus one (where breakpoints are slope changes). Thus, to analyze expressivity in 1D, it suffices to study the distribution of breakpoints. Our main result Theorem \ref{theorem:generalLbreakpoints} establishes the asymptotic expected number of breakpoints in randomly initialized DNNs.

\begin{definition}
    $R(T)$ denotes the random variable giving the number of breakpoints of $\Phi$, i.e., the number of points $x \in \mathbb{R}$ where $\Phi'(x)$ is undefined.
\end{definition}

\begin{theorem}\label{theorem:generalLbreakpoints}
    ~\vspace{-0.5cm}
    \begin{equation*}
        \lim_{\min(n_1, n_2,\dots,n_L) \to \infty}{\frac{\Exp[R(T)]}{\sum_{\ell=1}^L n_\ell}} = 1.
    \end{equation*}
\end{theorem}

\section{Proof of Main Result}\label{sec:proof}

To analyze breakpoints, we first observe that each breakpoint in a ReLU network is created when a ReLU activation is applied to a pre-activation that changes sign. Thus, it suffices to study the sign changes of the pre-activation functions which we do by partitioning intervals $[A, B]$ into small subintervals of length $\varepsilon > 0$ and examine the probability of zero crossings in each subinterval (Lemma~\ref{lem:signchange}). Taking $\varepsilon \to 0$ yields the expected number of sign changes Theorem~\ref{thm:zero-crossings}. Finally, to determine the number of breakpoints of the entire network, we analyze the probability that a breakpoint created at some intermediate layer propagates to the output. This yields the expected number of breakpoints, and therefore of the linear regions of the network as described in Theorem \ref{theorem:generalLbreakpoints}.  

\begin{proposition} \label{thm:cov-recursion} 
    For $\ell \in \indexset[2]{L+1}$, let $s_j^{(\ell)}(t)$ denote the pre–activation of neuron $j$ in layer $\ell$ at input $t \in \mathbb{R}$. Define
    \begin{equation*}
        C^{(\ell)}(u,v) \coloneq \Cov\!\big(s_j^{(\ell)}(u),s_j^{(\ell)}(v)\big), 
        \quad 
        V^{(\ell)}(u) \coloneq C^{(\ell)}(u,u).
    \end{equation*}
    Then, in the infinite–width limit $\min(n_1, \dots, n_L) \to \infty$, for $\ell \in \indexset[2]{L+1}$, the covariance satisfies the recursion
    \begin{equation*}
        C^{(\ell)}(u,v) = \sigma_b^2 + 2X \Big(\sin \theta_{\ell-1} + (\pi - \theta_{\ell-1}) \cos \theta_{\ell-1}\Big),
    \end{equation*}
    where
    \begin{align*}
        X & \coloneq \tfrac{1}{2\pi} \sqrt{V^{(\ell-1)}(u)V^{(\ell-1)}(v)}, \\
        \theta_{\ell-1} & \coloneq \arccos \rho_{\ell-1}(u,v), \\
        \rho_{\ell-1}(u, v) & \coloneq C^{(\ell-1)}(u,v)/\sqrt{V^{(\ell-1)}(u)V^{(\ell-1)}(v)}.
    \end{align*}
\end{proposition}
\begin{proof}
    Note that $C^{(\ell)}(u, v)$ and $V^{(\ell)}(u)$ are independent of any particular choice of $j$ since every weight and bias has the same distribution between the same layers. Let $Z_1 \coloneq s_i^{(\ell-1)}(u)$ and $Z_2 \coloneq s_i^{(\ell-1)}(v)$. 
    We then have the recursive formula
    \begin{align*}
        & C^{(\ell)}(u,v) \\
        & = \Cov\Bigg(\sum_{i=1}^{n_{\ell-1}} W_{ij}^{(\ell)} \relu\big(Z_1\big) + b_j^{(\ell)},\sum_{k=1}^{n_{\ell-1}} W_{kj}^{(\ell)} \relu\big(Z_2\big) + b_j^{(\ell)}\Bigg) \\
        & = \sum_{i,k=1}^{n_{\ell-1}}\Cov\Big(W_{ij}^{(\ell)} \relu\big(Z_1\big),W_{kj}^{(\ell)} \relu\big(Z_2\big)\Big)+ \Var\!\big(b_j^{(\ell)}\big).
    \end{align*}
    For each $i$, the weights $W_{ij}^{(\ell)}$ are independent from $h(Z_1)$, $h(Z_2)$, and $W_{kj}^{(\ell)}$ for $k \neq i$, so since $\Exp[W_{ij}^{(\ell)}] = 0$ by definition, $C^{(\ell)}(u, v)$ simplifies to
    \begin{align*}
        C^{(\ell)}(u,v) & = \sum_{i=1}^{n_{\ell-1}}\Var(W_{ij}^{(\ell)})\Exp\left[\relu(Z_1) \relu(Z_2)\right] + \sigma_b^2 \\
        & = \frac{2}{n_{\ell-1}} \cdot n_{\ell - 1} \Exp\left[\relu(Z_1) \relu(Z_2)\right] + \sigma_b^2.
    \end{align*}

    $Z_1$ and $Z_2$ are Gaussian, since in the infinite–width limit the pre-activations form a Gaussian process and thus the pre-activation at any input $u$ is Gaussian \cite{lee2017deep}. Corollary 3.1 of \cite{li2009gaussian} implies\footnote{Using the fact that $\max(x, 0) = (x + |x|)/2$ for any $x \in \mathbb{R}$.}
    \begin{equation*}
        \Exp[\relu(Z_1) \relu(Z_2)]
        = X \left(\sin \theta_{\ell-1} + (\pi - \theta_{\ell-1}) \cos \theta_{\ell-1}\right),
    \end{equation*}
    where $X$, $\theta_{\ell-1}$, and $\rho_{\ell-1}(u, v)$ are as in the statement of Proposition~\ref{thm:cov-recursion}. 
    Hence the covariance recursion is
    \begin{equation*}
        C^{(\ell)}(u, v) = \sigma_b^2 + 2X \left(\sin \theta_{\ell-1} + (\pi - \theta_{\ell-1}) \cos \theta_{\ell-1}\right).
    \end{equation*}
\end{proof}

\begin{corollary} \label{cor:variance-growth}
    For $\ell \in \indexset[2]{L+1}$, in the infinite–width limit, the variance recursion when $u = v$ is $V^{(\ell)}(u) = V^{(\ell-1)}(u) + \sigma_b^2$. Since $V^{(1)}(u) = 2u^2 + \sigma_b^2$, it follows that $V^{(\ell)}(u) = 2u^2 + \ell \sigma_b^2$.
\end{corollary}

\begin{proposition} \label{thm:corr-expansion}
    In the infinite–width limit,
    \begin{equation*}
        \lim_{\min(n_1, \dots, n_L) \to \infty}{\frac{1 - \rho_\ell(x, x + \varepsilon)}{\varepsilon^2}} = \frac{\ell \sigma_b^2}{\big(2x^2 + \ell \sigma_b^2\big)^2}.
    \end{equation*}
    Moreover
    \begin{equation*}
        \rho_\ell(x, x + \varepsilon)
        = 1 - \frac{\ell \sigma_b^2}{(2x^2 + \ell \sigma_b^2)^2} \,\varepsilon^2
        + \BigO(\varepsilon^3).
    \end{equation*}
\end{proposition}
\begin{proof}
    Write
    \begin{align*}
        V & \coloneq V^{(\ell-1)}(x) = 2x^2 + (\ell - 1) \sigma_b^2, \\
        V_y & \coloneq V^{(\ell-1)}(x + \varepsilon) 
         = V + 4x \varepsilon + 2\varepsilon^2.
    \end{align*}
    Define $A_\ell = A_\ell(x) \coloneq \ell\sigma_b^2/\left(2x^2 + \ell\sigma_b^2\right)^2$. We show by induction on $\ell$ that
    \begin{equation*}
        \rho_\ell(x, x + \varepsilon) =
        1 - A_\ell \varepsilon^2 + \BigO(\varepsilon^3) = 1 - \frac{\ell\sigma_b^2}{\left(2x^2 + \ell \sigma_b^2\right)^2} \varepsilon^2
        + \BigO(\varepsilon^3).
    \end{equation*}

    \emph{Base case ($\ell = 1$).} 
    By Proposition~\ref{thm:cov-recursion} and Corollary~\ref{cor:variance-growth}, 
    \begin{equation*}
        V^{(1)}(t) = 2t^2 + \sigma_b^2, \quad 
        C^{(1)}(x, x + \varepsilon) = 2x (x + \varepsilon) + \sigma_b^2.
    \end{equation*}
    Thus
    \begin{equation*}
        \rho_1(x, x + \varepsilon)
        = \frac{2x (x + \varepsilon) + \sigma_b^2}{\sqrt{(2x^2 + \sigma_b^2) \big(2(x + \varepsilon)^2 + \sigma_b^2\big)}}.
    \end{equation*}
    A Taylor expansion in $\varepsilon$ about $0$ gives\footnote{\label{Taylor expansion appendix footnote} See Appendix~\ref{Taylor expansion derivations appendix} for derivations.}
    \begin{equation*}
        \rho_1(x, x + \varepsilon)
        = 1 - \frac{\sigma_b^2}{(2x^2 + \sigma_b^2)^2} \varepsilon^2 + \BigO(\varepsilon^3),
    \end{equation*}
    so the base case holds with $A_1 = \sigma_b^2/\big(2x^2 + \sigma_b^2\big)^2$.

    \emph{Inductive step.} 
    Now fix $\ell \ge 2$ and assume $\rho_{\ell - 1}(x, x + \varepsilon) = 1 - A_{\ell-1} \varepsilon^2 + \BigO(\varepsilon^3)$. Then Taylor expanding in $\varepsilon$ about $0$ gives the following three Taylor expansions${}^\text{\ref{Taylor expansion appendix footnote}}$
    \begin{gather*}
        \theta_{\ell-1} = \arccos\big(\rho_{\ell-1}(x, x + \varepsilon)\big)
        = \sqrt{2A_{\ell-1}} \varepsilon + \BigO(\varepsilon^2), \\
        \frac{1}{\pi}\Big(\sin \theta_{\ell-1} + (\pi - \theta_{\ell-1}) \cos \theta_{\ell-1}\Big)
        = 1 - A_{\ell-1} \varepsilon^2 + \BigO(\varepsilon^3), \\
        \sqrt{V V_y}
        = V + 2x \varepsilon + \left(1 - \frac{2x^2}{V}\right) \varepsilon^2 + \BigO(\varepsilon^3).
    \end{gather*}
    Substituting these into the covariance recursion from Proposition~\ref{thm:cov-recursion} gives the expansion of $C^{(\ell)}(x, x + \varepsilon)$,
    \begin{equation} \label{series division numerator}
        \sigma_b^2 + V + 2x \varepsilon + \left(1-\tfrac{2x^2}{V} - A_{\ell-1} V\right) \varepsilon^2 + \BigO(\varepsilon^3). \tag*{(\textasteriskcentered)}
    \end{equation}
    Another Taylor expansion in $\varepsilon$ about $0$ gives${}^\text{\ref{Taylor expansion appendix footnote}}$
    \begin{align}
        \nonumber & \sqrt{V^{(\ell)}(x) V^{(\ell)}(x + \varepsilon)}
        = \sqrt{(V + \sigma_b^2) (V_y + \sigma_b^2)} \\ 
        & = (V + \sigma_b^2) + 2x \varepsilon + \left(1 - \frac{2x^2}{V + \sigma_b^2}\right) \varepsilon^2 + \BigO(\varepsilon^3). \label{series division denominator} \tag*{(\textasteriskcentered\textasteriskcentered)}
    \end{align}
    Divide the two series \ref{series division numerator} and \ref{series division denominator} to obtain${}^\text{\ref{Taylor expansion appendix footnote}}$
    \begin{equation*}
        \rho_\ell(x, x + \varepsilon)
        = 1 - \left(\frac{A_{\ell-1} V}{V + \sigma_b^2}
        + \frac{2\sigma_b^2x^2}{V (V + \sigma_b^2)^2}\right) \varepsilon^2 + \BigO(\varepsilon^3).
    \end{equation*}
    But then $V = 2x^2 + (\ell - 1) \sigma_b^2$ implies
    \begin{align*}
        \frac{A_{\ell-1} V}{V + \sigma_b^2} + \frac{2\sigma_b^2 x^2}{V (V + \sigma_b^2)^2} 
        & = \frac{(\ell - 1) \sigma_b^2}{V (V + \sigma_b^2)}
        + \frac{2 \sigma_b^2 x^2}{V (V + \sigma_b^2)^2} \\
        & = \frac{\sigma_b^2 \big((\ell - 1) (V + \sigma_b^2) + 2x^2\big)}{V (V + \sigma_b^2)^2} \\
        & = \frac{\ell V \sigma_b^2}{V (V + \sigma_b^2)^2} = A_\ell.
    \end{align*}
\end{proof}

\begin{lemma}\label{lem:angle}
    $\theta_\ell(x, x + \varepsilon) = \frac{\sqrt{2\ell} \sigma_b}{2x^2 + \ell \sigma_b^2} \varepsilon + \BigO(\varepsilon^3)$.
\end{lemma}
\begin{proof}
    Using the asymptotic expansion $\arccos(1 - t) = \sqrt{2t} + \BigO(t^{3/2})$ (see Appendix~\ref{appendix expansion of theta ell minus one}) with 
    \begin{equation*}
        t \coloneq 1 - \rho_\ell(x, x + \varepsilon) = \frac{\ell \sigma_b^2}{(2x^2 + \ell \sigma_b^2)^2} \varepsilon^2 + \BigO(\varepsilon^3)
    \end{equation*}
    gives the desired expansion.
\end{proof}

Now we turn to looking at sign changes.

\begin{lemma} \label{lem:signchange}
    Let $G^{(\ell)}(x) \coloneq \indicator{s^{(\ell)}_j(x) s^{(\ell)}_j(x + \varepsilon) < 0}$. Then
    \begin{equation*}
        \Exp[G^{(\ell)}(x)] = \frac{\theta_{\ell}(x, x + \varepsilon)}{\pi} + \BigO(\varepsilon^3) = \frac{1}{\pi} \frac{\sqrt{2 \ell} \sigma_b}{2x^2 + \ell \sigma_b^2} \varepsilon + \BigO(\varepsilon^3).
    \end{equation*}
\end{lemma}
\begin{proof}
    Since $s^{(\ell)}_j(x)$ and $s^{(\ell)}_j(x + \varepsilon)$ are mean-zero Gaussians (pre-activations are Gaussian at each fixed input) with correlation $\cos \theta$, by \cite{li2009gaussian} $\Pr(s^{(\ell)}_j(x) s^{(\ell)}_j(x + \varepsilon) < 0) = \theta_{\ell}/\pi$. Applying Lemma~\ref{lem:angle} concludes the proof.
\end{proof}

\begin{theorem}\label{thm:zero-crossings}
    The expected number of zero crossings defined as $N_0^{(\ell)}([A,B])$ of any pre-activation $s^{(\ell)}_j$ in layer $\ell \in \indexset[2]{L+1}$ on $[A, B]$ is
    \begin{align*}
        \Exp[N_0^{(\ell)}([A, B])] & = \frac{1}{\pi} \Bigg[
        \arctan\Big(\frac{\sqrt{2}B}{\sqrt{\ell}\sigma_b}\Big)
        - \arctan\Big(\frac{\sqrt{2}A}{\sqrt{\ell}\sigma_b}\Big)
        \Bigg],
    \end{align*}
\end{theorem}
\begin{proof}
    Partition $[A, B]$ into subintervals $I_j \coloneq [x_j, x_{j+1}]$ of length $0 < \varepsilon < 1$ chosen so that $N \coloneq (B - A)/\varepsilon \in \mathbb{N}$, with $x_j \coloneq A + j\varepsilon$ for each $j \in \indexset[0]{N - 1}$. Let $G^{(\ell)}(x)$ be as in Lemma~\ref{lem:signchange} and let $\lambda_\ell(x_j; \sigma_b) \coloneq \frac{1}{\pi} \frac{\sqrt{2 \ell} \sigma_b}{2x^2 + \ell \sigma_b^2}$. By Lemma~\ref{lem:signchange}, uniformly on compact sets, $\Exp[G^{(\ell)}(x_i)] = \lambda_\ell(x_i; \sigma_b) \varepsilon + \BigO(\varepsilon^3)$.
    Hence the expected total number of sign–change intervals is
    \begin{align*}
        \mathbb{E}\left[\sum_{i=0}^{N-1} G^{(\ell)}(x_i)\right]
        & = \sum_{i=0}^{N-1}{\lambda_\ell(x_i; \sigma_b) \varepsilon} + \BigO((B-A) \varepsilon^2).
    \end{align*}
    Letting $\varepsilon\to 0$ gives the Riemann–sum limit
    \begin{align*}
        \lim_{\varepsilon \to 0}{\mathbb{E}\left[\sum_{i=0}^{N-1} G^{(\ell)}(x_i)\right]}
&= \int_A^B{\lambda_\ell(x;\sigma_b)\,dx} \\
&= \frac{1}{\pi}\int_A^B{\frac{\sqrt{2\ell}\sigma_b}{2x^2+\ell\sigma_b^2}\,dx}.
    \end{align*}
    Since the number of sign–change intervals converges to the number of zero crossings as $\varepsilon \to 0$, the same limit equals $\Exp[N_0^{(\ell)}([A, B])]$. Evaluating the integral gives
    \begin{equation*}
        \int_A^B{\lambda_\ell(x; \sigma_b)\,dx} = \frac{1}{\pi} \arctan\big(\frac{\sqrt{2}x}{\sqrt{\ell} \sigma_b}\big) \Bigg|_A^B.
    \end{equation*}
\end{proof}

\begin{corollary}\label{cor:global}
    The expected number of breakpoints created by applying a ReLU to any pre-activation beyond the first layer is $\Exp\big[N_0^{(\ell)}([-\infty, \infty])\big] = 1$.
\end{corollary}

\begin{lemma}\label{prop:firstlayer}
    Each neuron in the first hidden layer contributes exactly one breakpoint at location $\frac{-b_i^{(1)}}{W_{1i}^{(1)}}$ for $1 \leq i \leq n_1$. Hence the total number of breakpoints contributed by the first layer is $n_1$.
\end{lemma}
\begin{proof}
    Immediate.
\end{proof}

\begin{proof}[Proof of Theorem \ref{theorem:generalLbreakpoints}]
    Let $X_{\ell, i}$ denote the number of breakpoints created at layer $\ell$ by applying the activation at the $i^{\text{th}}$ neuron. Because weights and biases are independent, mean–zero, and symmetric, a breakpoint created at the $i$-th neuron of layer $\ell$ is, for each unit $j$ in layer $\ell + 1$, preserved by $\relu(s^{(\ell+1)}_j)$ with probability $1/2$. Independence across $j$ then gives that the probability no unit in layer $\ell+1$ preserves the breakpoint equals $2^{-n_{\ell+1}}$, so the propagation probability to the next layer is $1 - 2^{-n_{\ell+1}} \to 1$ as $n_{\ell+1} \to \infty$. Iterating over layers shows that the output breakpoint count differs from the total number created by a term $\littleo\big(\sum_{\ell=1}^L n_\ell\big)$. Hence
    \begin{equation*}
        \Exp[R(T)] = \sum_{\ell=1}^L{\sum_{i=1}^{n_\ell}\Exp[X_{\ell,i}]} + \littleo\Big(\sum_{\ell=1}^L{n_\ell}\Big).
    \end{equation*}
    By Corollary~\ref{cor:global}, $\Exp[X_{\ell, i}] = 1$ for $\ell > 1$, and by Lemma~\ref{prop:firstlayer}, $\Exp[X_{1, i}] = 1$. Therefore
    \begin{align*}
    \Exp[R(T)]
    & =  \sum_{\ell=1}^L{n_\ell}
       + \littleo\Big(\sum_{\ell=1}^L{n_\ell}\Big).
    \end{align*}
    
    Equivalently, by dividing by $\sum_{\ell=1}^L{n_\ell}$, we find
    \begin{equation*}
        \lim_{\min(n_1, n_2,\dots,n_L) \to \infty}{\frac{\Exp[R(T)]}{\sum_{\ell=1}^L n_\ell}} = 1.
    \end{equation*}
\end{proof}

\section{Conclusion}

This work presents new results concerning the expected number of linear regions of a randomly initialized ReLU DNN in the case where both the input and output dimensions are one, as well as proposes a new, architecture-independent formulation of sparsity in the case of a known target function. 

More work is needed to expand these results to higher-dimensional cases. Another direction for future work is to determine the asymptotic distribution of $R(T)$ beyond just its expected value. Concerning the new proposed definition for function-adaptive sparsity, future research may focus on applying it to more concrete examples to better understand its significance and applicability.

\section*{Acknowledgment}

The authors wish to acknowledge the following individuals for their contributions and support: Daniel Andersen,
LaToya Anderson, Bill Arcand, Sean Atkins, Bill Bergeron,
Chris Berardi, Bob Bond, Alex Bonn, Bill Cashman, K
Claffy, Tim Davis, Chris Demchak, Alan Edelman, Peter
Fisher, Jeff Gottschalk, Thomas Hardjono, Chris Hill, Michael
Houle, Michael Jones, Charles Leiserson, Piotr Luszczek,
Kirsten Malvey, Peter Michaleas, Lauren Milechin, Chasen
Milner, Sanjeev Mohindra, Guillermo Morales, Julie Mullen,
Heidi Perry, Sandeep Pisharody, Christian Prothmann, Andrew
Prout, Steve Rejto, Albert Reuther, Antonio Rosa, Scott Ruppel, Daniela Rus, Mark Sherman, Scott Weed, Charles Yee,
Marc Zissman.

\bibliographystyle{ieeetr}
\bibliography{References}

\appendices
\appendices

\section{Derivations for Taylor Expansions} \label{Taylor expansion derivations appendix}
We collect here the Taylor expansions used in \S\ref{sec:proof}.

\subsection{Expansion of $\rho_1(x, x + \varepsilon)$}
\label{appendix expansion of rho1}

Let $A \coloneq 2x^2 + \sigma_b^2$ and $u \coloneq \tfrac{4x}{A} \varepsilon + \tfrac{2}{A} \varepsilon^2$. Then
\begin{align*}
    \rho_1(x, x + \varepsilon)
    & = \frac{2x (x + \varepsilon) + \sigma_b^2}{\sqrt{(2x^2 + \sigma_b^2) (2 (x + \varepsilon)^2 + \sigma_b^2)}} \\ 
    & =\frac{A + 2x \varepsilon}{A\sqrt{1 + \tfrac{4x}{A} \varepsilon + \tfrac{2}{A} \varepsilon^2}} \\
    & =\Big(1 + \tfrac{2x}{A} \varepsilon\Big) (1+u)^{-\frac{1}{2}}.
\end{align*}
Using $(1+u)^{-1/2} = 1 - \tfrac{1}{2} u + \tfrac{3}{8} u^2 + \BigO(u^3)$,
\begin{equation*}
    \rho_1(x,x+\varepsilon) = \Big(1 + \tfrac{2x}{A} \varepsilon\Big)
    \Big(1 - \tfrac{2x}{A} \varepsilon - \tfrac{1}{A} \varepsilon^2 + \tfrac{6x^2}{A^2} \varepsilon^2 + \BigO(\varepsilon^3)\Big).
\end{equation*}
When carrying out the multiplication, the linear terms cancel, the quadratic terms yield the coefficient $-\sigma_b^2/A^2$, and the higher power terms may be absorbed into $\BigO(\varepsilon^3)$ to give
\begin{equation*}
    \rho_1(x, x + \varepsilon) = 1 - \frac{\sigma_b^2}{(2x^2 + \sigma_b^2)^2} \varepsilon^2 + \BigO(\varepsilon^3).
\end{equation*}

\subsection{Expansion of $\theta_{\ell-1}$}
\label{appendix expansion of theta ell minus one}

By the induction hypothesis, $\rho_{\ell-1}(x, x + \varepsilon) = 1 - A_{\ell-1} \varepsilon^2 + \BigO(\varepsilon^3)$. Next, consider the function $f \colon t \mapsto \arccos(1 - t)$; although $f$ cannot be Taylor expanded in $t$ about $0$ (it is not differentiable at $t = 0$) it can be asymptotically expanded \cite{erdelyi1956asymptotic} by using the observation and standard Taylor expansion
\begin{equation*}
    \frac{d}{dt}[\arccos(1 - t)] = \frac{1}{\sqrt{2t}\sqrt{1 - t/2}} = \frac{1}{\sqrt{2t}} \left(1 + \BigO(t)\right).
\end{equation*}
Integrating yields $\arccos(1 - t) = \sqrt{2t} + \BigO(t^{3/2})$. Letting $t \coloneq 1 - \rho_{\ell-1} = A_{\ell - 1} \varepsilon^2 + \BigO(\varepsilon^3)$ then gives
\begin{equation*}
    \theta_{\ell-1} = \sqrt{2A_{\ell-1}}\,\varepsilon + \BigO(\varepsilon^2).
\end{equation*}

\subsection{Expansion of $\frac{1}{\pi}(\sin \theta_{\ell-1} + (\pi - \theta_{\ell-1}) \cos \theta_{\ell-1})$}
\label{appendix expansion of sine cosine expression}

Taylor expanding $\sin \theta_{\ell-1}$ and $\cos \theta_{\ell-1}$ in $\theta_{\ell-1}$ about $0$, then substituting the Taylor expansion of $\theta_{\ell-1}$ gives
\begin{align*}
    & \frac{1}{\pi} \big(\sin \theta_{\ell-1} + (\pi - \theta_{\ell-1}) \cos \theta_{\ell-1}\big) \\
    & = 1 - \tfrac{1}{2} \theta_{\ell-1}^2 + \BigO(\theta_{\ell-1}^3) 
    = 1 - A_{\ell-1} \varepsilon^2 + \BigO(\varepsilon^3).
\end{align*}

\subsection{Expansion of $\sqrt{V V_y}$}
\label{appendix expansion of sqrt V Vy}

Using the standard Taylor expansion $\sqrt{a + t} = \sqrt{a} + \frac{1}{2\sqrt{a}} t - \frac{1}{4a^{3/2}} t^2 + \BigO(t^3)$ shows $\sqrt{V V_y}$ can be expanded as
\begin{equation*}
    \sqrt{V^2 + (4xV \varepsilon + 2V \varepsilon^2)} = V + 2x \varepsilon + \Big(1 - \tfrac{2x^2}{V}\Big) \varepsilon^2 + \BigO(\varepsilon^3).
\end{equation*}

\subsection{Expansion of $C^{(\ell)}(x, x+ \varepsilon)$}
\label{appendix expansion of C ell}

Substituting prior expansions into the covariance recursion of Proposition~\ref{thm:cov-recursion} for $C^{(\ell)}(x, x + \varepsilon)$ gives
\begin{equation*}
    \sigma_b^2 + V + 2x \varepsilon + \Big(1 - \tfrac{2x^2}{V} - A_{\ell-1} V\Big) \varepsilon^2 + \BigO(\varepsilon^3).
\end{equation*}

\subsection{Expansion of $\sqrt{V^{(\ell)}(x) V^{(\ell)}(x + \varepsilon)}$}
\label{appendix expansion of sqrt V ell product}

Analogous to the Taylor expansion of $\sqrt{V V_y}$.

\subsection*{Expansion of $\rho_\ell(x, x + \varepsilon)$}
\label{appendix expansion of rho ell}

Let $W \coloneq V + \sigma_b^2$, $a \coloneq 1 - \tfrac{2x^2}{V} - A_{\ell-1} V$, and $b \coloneq 1 - \tfrac{2x^2}{W}$. Taking the ratio of the Taylor expansions in $\varepsilon$ about $0$ for $\rho_\ell(x, x + \varepsilon) = C^{(\ell)}(x, x + \varepsilon)/\sqrt{V^{(\ell)}(x) V^{(\ell)}(x + \varepsilon)}$ gives
\begin{equation*}
    \rho_\ell(x, x + \varepsilon)
    =\frac{1 + \tfrac{2x}{W} \varepsilon + \tfrac{a}{W} \varepsilon^2}{1 + \tfrac{2x}{W} \varepsilon + \tfrac{b}{W} \varepsilon^2} + \BigO(\varepsilon^3).
\end{equation*}
Expanding $(1 + z)^{-1} = 1 - z + z^2 + \BigO(z^3)$, all linear terms cancel and
\begin{equation*}
    \rho_\ell(x, x + \varepsilon) = 1 - \tfrac{b - a}{W} \varepsilon^2 + \BigO(\varepsilon^3).
\end{equation*}
Since $b - a = A_{\ell-1} V + 2x^2 (\tfrac{1}{V} - \tfrac{1}{W})$,
\begin{equation*}
    A_\ell = \frac{A_{\ell-1}V}{V+\sigma_b^2}+\frac{2\sigma_b^2x^2}{V(V+\sigma_b^2)^2}.
\end{equation*}

\end{document}